\documentclass{article}

\usepackage{microtype}
\usepackage{graphicx}
\usepackage{subfigure}
\usepackage{booktabs} 

\usepackage{enumitem}

\usepackage{hyperref}
\usepackage{wrapfig}
\PassOptionsToPackage{numbers, compress}{natbib}
\usepackage[preprint]{neurips_2023}

\usepackage{amsmath}
\usepackage{amssymb}
\usepackage{mathtools}
\usepackage{amsthm}

\usepackage{algorithm}
\usepackage{algorithmic}
\usepackage[capitalize,noabbrev]{cleveref}

\theoremstyle{plain}
\newtheorem{theorem}{Theorem}[section]

\theoremstyle{definition}
\newtheorem{definition}[theorem]{Definition}
\newtheorem{assumption}[theorem]{Assumption}
\theoremstyle{remark}

\usepackage[textsize=tiny]{todonotes}

\begin{document}
\newcommand{\figureFramework}{
\begin{figure*}[!t]
    \centering
    \includegraphics[width=\textwidth]{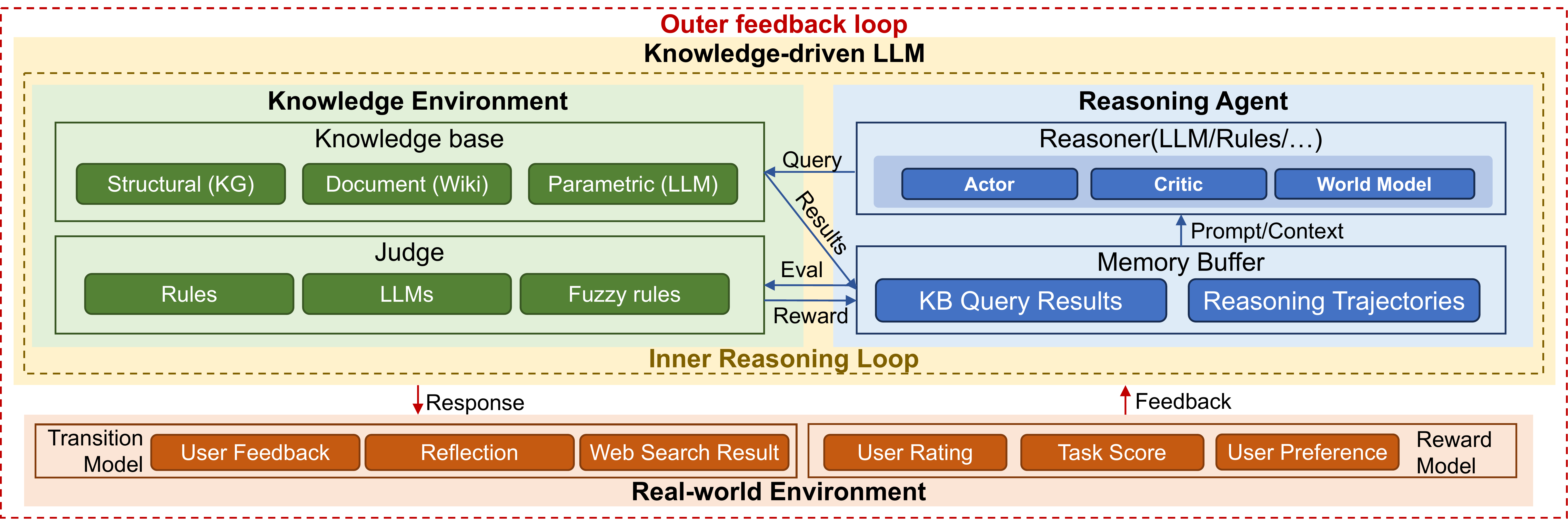}
    \caption{Knowledge-driven LLM agent framework. This framework consists of two levels of interaction. The inner level as shown in the left part corresponds the iterative reasoning process between LLM and external knowledge base. The outer level corresponds to the interaction between LLM agent and real world, where feedbacks is provided to help the agent improve itself.}
    \label{fig:2-level_framework}
\end{figure*}
}
\newcommand{\algInnerLoop}{
\begin{algorithm}
    \caption{The inner reasoning loop}
    \label{alg:inner_loop}
    \begin{algorithmic}[1]
        \item[\textbf{Input:}] User question $q$, max reasoning step $T$, reward threshold $R$
        \item[\textbf{Output:}] Response $a$
        \STATE $t \leftarrow 0, \mathcal{D}_t \leftarrow {q}, terminal \leftarrow False$
        \STATE $s_t \leftarrow \{q, \emptyset, \emptyset\}$
        \WHILE{$terminal \neq True$}
            \STATE $a^s_t, a^q_t \leftarrow \text{Reason}(\mathcal{D}_t)$
            \STATE $s^p_{t+1}, s^n_{t+1} \leftarrow \text{Select\_and\_query}(s_{t}, a^s_t)$
            \STATE $s_{t+1} \leftarrow (q, s^p_{t+1}, s^n_{t+1})$
            \STATE $r_t \leftarrow \text{Judge}(s_{t+1})$
            \STATE $\mathcal{D}_{t+1} \leftarrow \mathcal{D}_t \cup \{(s_t, a_t, r_t, s_{t+1})\}$
            \STATE $t \leftarrow t+1$
            \IF{$r_t \ge R \OR t \ge T$}
                \STATE $terminal \leftarrow True$
            \ENDIF
        \ENDWHILE
        \STATE \textbf{return} $\text{summarize}(\mathcal{D}_t)$
    \end{algorithmic}
\end{algorithm}
}

\newcommand{\algBayesianMBP}{
\begin{algorithm}[!t]
    \caption{Reason (LLM-driven Bayesian model-based planning with tree search)}
    \label{alg:bayesian_mbp}
    \begin{algorithmic}[1]
        \item[\textbf{Input:}] Memory buffer $\mathcal{D}_t$ at time $t$, model prompt template $pt_m$, critic prompt template $pt_c$, actor prompt template $pt_a$, planning lookahead $U$, beam width $W$, beam proposal number $N$
        \item[\textbf{Output:}] Action $a_t$
        \STATE \texttt{Model} $\leftarrow$ LLM prompted with $pt_m$ and $\mathcal{D}_t$
        \STATE \texttt{Critic} $\leftarrow$ LLM prompted with $pt_c$ and $\mathcal{D}_t$
        \STATE \texttt{Actor} $\leftarrow$ LLM prompted with $pt_a$ and $\mathcal{D}_t$
        \FOR{$i=0$ to $U-1$}
            \STATE \texttt{Actor} propose $N$ possible actions
            \STATE \texttt{Model} acts as a fake knowledge base that executes these actions and generate next states
            \STATE \texttt{Critic} evaluate these actions and select top-$W$ ones
            \STATE Update prompts for \texttt{Model}, \texttt{Actor}, \texttt{Critic} with the $W$ new trajectories
        \ENDFOR
        \STATE \textbf{return} The first action on the trajectory with the highest value among $W$ trajectories
    \end{algorithmic}
\end{algorithm}
}

\newcommand{\algAdaptedInnerLoop}{
\begin{algorithm}[!t]
    \caption{The adapted inner reasoning loop}
    \label{alg:adapted_inner_loop}
    \begin{algorithmic}[1]
        \item[\textbf{Input:}] User question $q$, max reasoning step $T$, reward threshold $R$
        \item[\textbf{Output:}] Response $a$
        \STATE $t \leftarrow 0, k \leftarrow 0, \mathcal{D}_t \leftarrow {q}, terminal \leftarrow False$
        \STATE $s_k \leftarrow \{q, \emptyset, \emptyset\}$
        \WHILE{$terminal \neq True$}
            \STATE $B_k \leftarrow \mathcal{D}_t$ \COMMENT{periodic update memory buffer for LLM}
            \REPEAT
                \STATE $a^s_t, a^q_t \leftarrow \text{Reason}(\mathcal{B}_t)$
                \STATE $s^p_{t+1}, s^n_{t+1} \leftarrow \text{Select\_and\_query}(s_{t}, a^s_t)$
                \STATE $s_{t+1} \leftarrow (q, s^p_{t+1}, s^n_{t+1})$
                \STATE $r_t \leftarrow \text{Judge}(s_{t+1})$
                \STATE $\mathcal{D}_{t+1} \leftarrow \mathcal{D}_t \cup \{(s_t, a_t, r_t, s_{t+1})\}$
                \STATE $t \leftarrow t+1$
                \IF{$r_t \ge R \OR t \ge T$}
                    \STATE $terminal \leftarrow True$
                \ENDIF
            \UNTIL{\texttt{Enough-NewInfo}}
            \STATE $k = k + 1$
        \ENDWHILE
        \STATE \textbf{return} $\text{summarize}(\mathcal{D}_t)$
    \end{algorithmic}
\end{algorithm}
}

\title{A Principled Framework for Knowledge-enhanced Large Language Model}
\author{%
  \text{Saizhuo Wang}$^{21}$\thanks{Work done during internship at IDEA Research}$\text{, Zhihan Liu}^{4} \text{, Zhaoran Wang}^4\text{, Jian Guo}^{13}$\thanks{Corresponding Author} \\
  $^1$IDEA Research, International Digital Economy Academy \\
  $^2$The Hong Kong University of Science and Technology\\
  $^3$The Hong Kong University of Science and Technology (Guangzhou) \\
  $^4$Northwestern University \\
  \texttt{swangeh@connect.ust.hk, zhihanliu2027@u.northwestern.edu, }\\
  \texttt{zhaoranwang@gmail.com, guojian@idea.edu.cn}
}
\maketitle

\begin{abstract}
Large Language Models (LLMs) are versatile, yet they often falter in tasks requiring deep and reliable reasoning due to issues like hallucinations, limiting their applicability in critical scenarios. This paper introduces a rigorously designed framework for creating LLMs that effectively anchor knowledge and employ a closed-loop reasoning process, enhancing their capability for in-depth analysis. We dissect the framework to illustrate the contribution of each component to the LLMs’ performance, offering a theoretical assurance of improved reasoning under well-defined assumptions.
\end{abstract}
\section{Introduction}\label{sec:intro}
Large language models (LLMs) have achieved notable success in natural language understanding and generation, laying the foundation for numerous LLM agents. These agents leverage LLMs as their cognitive nucleus, supplemented with various tools, memory, and reasoning mechanisms. This integration empowers them to conceptualize and act, thereby handling complex and pragmatic tasks. While proficient in routine tasks like answering questions and assisting with daily activities such as food orders, online purchases, and bookings, LLM agents often fall short in knowledge-dense tasks that require complex reasoning and inference. Such tasks demand intricate reasoning chains and high factual accuracy, challenges that LLMs face due to their reliance on precise prompt engineering, a constrained context window that cannot fully track extensive reasoning processes, and the tendency to generate unverified information, known as hallucinations.

Efforts to fortify factual accuracy have led to the integration of external knowledge bases, like knowledge graphs and document databases, and the enhancement of deep reasoning with iterative or recursive reasoning mechanisms. Nevertheless, these solutions tend to focus on individual aspects of the problem. Recognizing the need for a holistic approach, we introduce a cohesive, principle-driven framework to tackle the challenge of deep and responsible reasoning in LLMs. This framework systematically analyzes the intricacies of designing such agents, discussing the architectural decisions and what they mean for performance and functionality.

\section{Background}
\subsection{Knowledge-driven LLM and LLM Agent}
Current state-of-the-art large language models (LLMs) exhibit notable deficiencies. They are particularly prone to producing content that may not be factually accurate, commonly referred to as 'hallucinations' \cite{zhang_sirens_2023}. Additionally, updating the knowledge within these models is a resource-intensive process that typically requires retraining the entire model. To mitigate these limitations, integrating LLMs with external knowledge bases, such as a knowledge graph (KG), has emerged as a viable strategy. This integration enables LLMs to ground their responses in the verifiable facts stored in the KG, and updating knowledge becomes more manageable and cost-effective through modifications to the KG. A common method in knowledge-driven LLMs involves retrieving relevant information from the KG to provide context for the LLM's inference process. Leveraging LLMs' ability to learn in context \cite{dong_survey_2023}, the output is made to align with the provided knowledge context, enhancing consistency and reliability.

\subsection{LLM Inference as Implicit Bayesian Inference}\label{sec:bg_icl}
Recent research \citep{xie_explanation_2021, zhang_what_2023} suggests that the inference mechanisms of Transformer-based LLMs resemble implicit Bayesian inference. The process can be mathematically represented as:
\begin{equation}
p(y|x) = \int_\theta p(y|\theta, x) p(\theta|x) d\theta
\end{equation}
Here, $x$ denotes the given context, $y$ is the output from the LLM, and $\theta$ symbolizes the latent "concept" inferred from the context. This conceptual model delineates LLM's text generation process, which initially deduces the latent concept from the context ($p(\theta|x)$) and subsequently generates the output based on this deduced concept ($p(y|\theta, x)$). In the context of knowledge-driven LLMs, the "concept" encompasses the external knowledge derived from the knowledge base, which is then seamlessly integrated by the LLM with its pre-trained knowledge to formulate an accurate response.

\figureFramework

\section{Framework}\label{sec:framework}
Our framework introduces a two-tiered process for reasoning with LLM agents: the inner reasoning loop and the outer response-feedback loop. Initially, the inner loop addresses a user's natural language query by engaging the knowledge base and producing a response. Once a response is generated, the outer loop incorporates user feedback, which may suggest changes to the knowledge base, to refine the agent's future responses. This comprehensive approach ensures a dynamic evolution of the agent's capabilities, potentially extending across various domains that utilize iterative learning and feedback.
Our framework is primarily examined in the context of knowledge-based question-answering (KBQA) systems; however, its flexible architecture is also suitable for various domains that benefit from iterative learning and feedback loops. Additionally, although our framework is designed for iterative use, it can be adapted for one-shot applications, aligning with numerous methodologies such as those that inject knowledge into LLMs through fine-tuning.

\subsection{The Inner Reasoning Loop}\label{sec:framework_inner}
\algInnerLoop
The inner loop models the interaction between the reasoning agent—either a Large Language Model (LLM) or a rule-based system—and the knowledge base, which could take the form of a knowledge graph, a document trove, or another LLM. Each cycle involves the agent consulting the memory buffer, which begins with the user's question and is progressively enriched with data from the knowledge base, to identify the next query. When the knowledge base receives a query, it retrieves and supplies relevant information. Subsequently, a judge—implemented through rules or an LLM—assesses the accumulated information in the buffer to determine if it suffices to address the user's question. The loop concludes and formulates a response once it is determined that the gathered information is adequate or when a predefined stopping point is reached.

\paragraph{Components}
The inner loop is a sequence of exchanges between the \textbf{knowledge environment} and the \textbf{reasoning agent}. The \textbf{knowledge environment} is comprised of a \textit{knowledge base}, functioning as the transition model by processing queries and returning results, and a \textit{judge} module, which evaluates whether the reasoning agent's memory buffer contains adequate information to answer the user's question. The \textit{judge} may operate as either a prompted LLM or a set of predefined stopping rules.
The reasoning agent hosts a memory buffer that retains the history of reasoning trajectories and a reasoning module that devises actions based on the buffer's contents. At a given time step $ t $, the memory buffer $ \mathcal{D}_t $ encapsulates the reasoning paths derived from the LLM. The reasoning module, which may be an LLM or a suite of rules, inputs $ \mathcal{D}_t $ to generate an action.

\paragraph{MDP Formalization}

At the core of the inner reasoning loop is the interaction between reasoning agent and knowledge base, which we model as a Markov Decision Process (MDP). This MDP is defined by the tuple $ (\mathcal{S}, \mathcal{A}, T, r, \gamma) $, where each component is characterized as follows:

\begin{itemize}[leftmargin=1em]
  \item \textbf{State Space} $ \mathcal{S} $: The state at any time $ t $, denoted as $ s_t $, is a composite of the initial query $ q $, the cumulative reasoning path $ s^p_t $, and the newly acquired information $ s^n_t $ from the knowledge base. At the outset, $ s^p_{t_0} $ is an empty set, indicating the start of the reasoning process. $s_t$ that encapsulates all historical information explored up to time $t$, and we call it the information state \cite{liu_reason_2023}.
  
  \item \textbf{Action Space} $ \mathcal{A} $: The agent's actions, represented by $ a_t $, are twofold: $ a^s_t $ involves selecting informative segments from the current information state $ s^n_t $, and $ a^q_t $ concerns formulating queries to the knowledge base for further information.
  
  \item \textbf{Transition Function} $ T $: Defined as $ T(s'|s, a) $, this function represents the environment's dynamics, describing how the agent's selection of informative segments is a deterministic process, while querying the knowledge base is probabilistic, depending on the base's structure and contents.
  
  \item \textbf{Reward Function} $ r $: This function, $ r(s_t) $, assigns a quantitative value to the state $ s_t $ that reflects the adequacy of information within the state to address the initial query $ q $. It is normalized to the interval $[0, 1]$, where a score of 1 indicates a state with fully sufficient information.
  
  \item \textbf{Discount Factor} $ \gamma $: It defines the weight of future rewards in the agent's consideration, shaping the strategic depth of the reasoning process by emphasizing the importance of long-term outcomes. In practice, we use $\gamma \in (0, 1)$ to ensure that the value function $V^\pi_\theta(s)$ is bounded, which is important for analysis.
\end{itemize}

Given a policy $\pi: \mathcal{S} \mapsto \mathcal{A}$ and an environment parameterized by $\theta$, we define its value function and Q-function as follows:
\begin{equation}
    V^\pi_\theta(s) = \mathbb{E}\bigl[\sum_{t=0}^\infty \gamma^t r_\theta(s_t, a_t) | s_0=s\bigr], Q^\pi_\theta(s, a) = \mathbb{E}\bigl[\sum_{t=0}^\infty \gamma^t r_\theta(s_t, a_t)|s_0=s, a_0=a\bigr],
\end{equation}
where the expectation is taken over $a_t \sim \pi(s_t)$ and $s_{t+1} \sim T_\theta(\cdot | s_t, a_t)$ for all $t > 0$. Here $\theta$ can be considered as knowledge base contents. Specifically, $\theta$ affects the transition function $T_\theta(\cdot | s_t, a_t)$ by determining on the query results of a given query action, and the reward function $r_\theta(s)$ can be instantiated as an oracle based on knowledge base, such as a set of knowledge rules or a knowledge-enhanced LLM. The goal of the inner loop, therefore, is to learn a policy that maximizes $V^\pi_\theta(s)$ for all $s \in \mathcal{S}$ so as to find enough information for arbitrary user questions. 


\paragraph{Workflow}
The reasoning process within this MDP framework, as shown in Algorithm \ref{alg:inner_loop}, consists of the following iterative steps:

\begin{enumerate}[leftmargin=1em]
  \item \textbf{Reason}: The agent observes memory buffer $\mathcal{D}_t$, and generates an action $a_t$. The specific reasonoing process can differ across differnt reasoning agent choices, which we will discuss in detail in Sec. 4.
  \item \textbf{Select and Query}: The agent applies the action $a_t$ to the current information state, where $a^s_t$ filters relevant information and $a^q_t$ poses a new query to the knowledge base. The knowledge base, acting as the environment, responds by updating the state with new information, resulting in $s_{t+1}$. The forms of query action differ across different KB forms, being databse query for KGs, retrieval command for document bases, or prompts for LLMs.
  \item \textbf{Evaluate}: A reward $r(s_t, a_t)$ is assigned based on the informativeness of the new state in answering the query. If the judge is a set of rules, then the reward could be the extent of satisfaction. If the judge is a LLM, then the reward can be a relative ranking score or so, depending on how the judge LLM is prompted.
  \item \textbf{Terminate}: The reasoning loop concludes when a termination condition is met, either because the agent has obtained a state that sufficiently answers the query (reward threshold met) or when a predefined number of reasoning steps has been reached.
\end{enumerate}

The reasoning process unfolds as the agent, equipped with a prior belief $p(\theta)$ about the knowledge environment's parameters $\theta^*$, engages in a cycle of queries and responses. Initially, the agent faces significant uncertainty about the true parameter distribution $\theta^*$. With each interaction, the agent refines its understanding by updating the posterior distribution of $\theta$, which in turn informs more targeted and informed queries. As the agent accumulates knowledge, its posterior inference becomes increasingly accurate, more closely mirroring the true parameter distribution of the knowledge environment, thus enhancing performance. The agent's learning progress and policy refinement are quantified by Bayesian regret, which we expect to be sublinear in $T$, indicating that the agent's policy is progressively approaching optimality. The Bayesian regret is defined as:
\begin{equation}\label{equ:bayesian_regret}
\mathcal{R}(T) = \mathbb{E}_{\theta \sim p(\theta)} \bigl[\sum_{t=1}^T V^{\pi^*}_\theta (s_t) - V^{\pi^t}_\theta (s_t) \bigr],
\end{equation}
where the expectation is taken over the prior distribution of $\theta$, and $\pi^*$ is the optimal policy.
Here, $p(\theta)$ represents the agent's initial model of the environment, shaped by its prior experiences or pretraining. This prior serves as the baseline from which the agent begins its learning and adapts through experience.

\paragraph{Correspondence with Existing Methodologies}

Our framework integrates various established methods into four main categories, each unified by the concept of a reliable judge—implemented as either a sophisticated LLM or a set of rules—to ensure consistent assessment.

\begin{enumerate}[leftmargin=1em]
  \item \textbf{Knowledge Graph (KG) Predominant (KG-only)}: The knowledge base (KB) is instantiated as a KG, and the reasoning agent operates via a suite of (soft) logic rules or arithmetic operations. This configuration is akin to traditional symbolic KG reasoning methods or the embedding-based arithmetic operations found in neural KG reasoning paradigms.
  \item \textbf{LLM Predominant (LLM-only)}: Both the KB and the reasoning agent are represented by LLMs. Here, the reasoning agent solicits information from another LLM, which reflects the operational mechanism of existing methodologies such as the Chain-of-Thought (CoT) approach \citep{wei_chain--thought_2022}.
  \item \textbf{LLM $\oplus$ KG}: Here, a KG serves as the knowledge base while an LLM acts as the reasoning agent, typically concluding the reasoning in one iteration.
    \item \textbf{LLM $\otimes$ KG}: A hybrid model where a KG provides the knowledge base and an LLM acts as an iterative reasoning agent, working continuously until it satisfies the query's information requirements.
\end{enumerate}

These classifications highlight our framework's ability to incorporate and extend upon a diverse range of reasoning approaches, showcasing its modularity and adaptability within a cohesive system.

\subsection{The Outer Feedback Loop}
The outer loop functions as an interactive layer where the LLM agent engages with real-world inputs. The agent processes user queries and produces answers through reasoned analysis. Feedback from this interaction informs subsequent updates to the agent's knowledge, refining its future performance.
This external process can be either a single interaction or a continuous, iterative dialogue, tailored to the application's requirements. In a knowledge-based question-answering (KBQA) scenario, the user presents a question, and the agent responds. If the interaction is one-time, the user's feedback may simply be a rating that concludes the session. In an iterative model, the user's feedback includes both a rating and constructive comments, prompting the agent to refine its subsequent responses.

The iterative nature of this loop can also be modeled as an MDP, analogous to the inner loop, which allows for a structured analysis of the feedback mechanism. Here, the combined knowledge base and reasoning agent form the prior for model parameter distribution, which evolves as it assimilates feedback from the environment. This process enhances the LLM’s understanding of the real world, gradually honing its performance as feedback accumulates.
Practically, feedback integration can vary, ranging from direct edits to the knowledge base, fine-tuning of the LLM, or adjustments to the prompts used for context. These methods represent design choices that will be explored further in our discussion.
It should be noted that in certain applications, such as one-off question answering, the LLM's opportunity for learning from user feedback is limited, as the interaction with the real-world environment does not allow for ongoing adjustments.

\section{Analysis}\label{sec:analysis}
In this section we analyze the effect of design components in our framework. We will theoretically prove the query efficiency of the inner reasoning loop and qualitatively analyze the design choices of other components in this framework.

\algBayesianMBP
\algAdaptedInnerLoop
\subsection{Provable query efficiency}
Our framework predicates its convergence towards accurate responses on high-quality knowledge bases, iteratively reasoned. This iterative reasoning is exemplified by the adapted inner reasoning loop in Algorithm \ref{alg:adapted_inner_loop}, which underpins our framework's query efficiency. Our analysis and proof follows the RAFA framework \citep{liu_reason_2023}, and we will briefly recite it here.

Our analysis focus on the inner loop with LLM $\otimes$ KG paradigm, where the \textbf{Reason} (Sec. \ref{sec:framework_inner}) step employs Bayesian model-based planning, with example provided in Algorithm \ref{alg:bayesian_mbp}. Updates to the LLM’s context within this configuration are periodic and significant, transitioning the agent's understanding from a prior environmental parameters distribution $p(\theta)$ to a refined posterior $p_t(\theta|\mathcal{D}_t)$ with each iteration. This process, hinged on following two critical assumptions, enhances the model-based planning and is assessed through Bayesian regret.
\begin{assumption}\label{assump:prior_same}
The distribution of the pretraining data for the reasoner LLM mirrors the knowledge distribution that underlies the knowledge base.
\end{assumption}

\begin{assumption}\label{assump:oracle_judge}
The inner loop judge delivers unerring reward signals about information sufficiency.
\end{assumption}

Assumption \ref{assump:prior_same} ensures no discrepancy in the inner loop between the prior distribution provided by LLM and the knowledge distribution. Assumption \ref{assump:oracle_judge} ensures the alignment between knowledge base and judge.
With these assumptions and an optimal planning algorithm in place, Bayesian regret is governed by the precision of the posterior distribution $p_t(\theta|\mathcal{D}_t)$.
Although Bayesian regret inherently accumulates over time, it does so at a sublinear rate, indicating that the average regret per iteration approaches zero, reflecting the policy's gradual improvement towards optimality.
This relationship between information assimilation and model precision underscores the agent's consistent advancement towards the ideal policy.
We will delve into planning optimality and information gains, concluding with the implications on the sublinear nature of Bayesian regret.

\subsubsection{Planning Optimality}
We will assume that LLM reasoning is performing poesterior inference of model parameter, and then prove the optimality of our model-based planning algorithm (example in Algorithm \ref{alg:bayesian_mbp}). Combining these two provides guarantee for our planning algorithm, and offloads policy suboptimality to information gap as we will analyze in Sec. \ref{sec:analysis_info}.

\begin{assumption}\label{assump:bayesian_inference}
    LLM-estimated knowledge model parameter a posteriori $p_t(\theta)$ essentially follows $\mathbb{P}(\cdot | \mathcal{D}_t)$
\end{assumption}
Assumption \ref{assump:bayesian_inference} says that LLM performs implicit Bayesian inference on knowledge model parameters, which echos Sec. \ref{sec:bg_icl} and is verified in \cite{xie_explanation_2021}.

\begin{definition}
    ($\epsilon$-Optimality) Under an environment parametrized by $\theta$, a policy $\pi$ satisfies $\epsilon$-optimality if the following condition holds with high probability
    \begin{equation}\label{equ:eps_optim}
        \max_{s} V^{\text{PL}^*(\theta)}_\theta(s) - V^\pi_\theta(s) \le \epsilon,
    \end{equation}
\end{definition}
where $\text{PL}^*(\theta)$ denotes the optimal planner under environment $\theta$.
For planning algorithms, $\epsilon$ is usually related to the lookhead horizon $U$, since as $U$ increases, the performance gap decreases. This essentially tells us that we can increase lookahead length to enhance planning optimality, which is essentially about trading inference cost/latency with accuracy. Planning algorithm such as value iteration, tree search (Algorithm \ref{alg:bayesian_mbp}), and MCTS can usually assure $\epsilon$-optimality with proper lookheads.

Given that the planning algorithm is based on the posterior estimation of model parameter $\theta_t$, which depends on effective information up to time $t$, we say that the reasoning LLM induces a corresponding policy $\pi^t$ at time step $t$ that evolves as information accumulates.

\subsubsection{Information Gain}\label{sec:analysis_info}
Consider the entropy of the posterior distribution $p_t(\theta)$ at time t:
\begin{equation}
    H_t = - \int p_t(\theta) \log p_t(\theta) d\theta
\end{equation}
$H_t$ is expected to decrease as $t$ increases, meaning that the uncertainty in environment estimation reduces as more observations is attained. Meanwhile, the uncertainty in posterior distribution will be reflected on the estimation error of value function. To reflect this process, we introduce the idea of information coefficient $\Gamma_{t^\dagger}(\delta)$ \citep{abbasi-yadkori_bayesian_2015}, which is the minimum value that makes the following inequality holds for all $t \in \{t^\dagger+1, ..., T-1\}$ with probability at least $1 - \delta$, if $H_{t^\dagger} - H_t \leq \log2$:
\begin{equation}\label{equ:info_coeff}
    | (r_{\theta^*} - r_{\theta_t})(s_t, a_t) + ((P_{\theta^*} - P_{\theta_t})V)(s_t, a_t) | \le \Gamma_{t^\dagger}(\delta) \sqrt{I(\theta;\xi_{t+1}|D_t)},
\end{equation}
where $I(\theta;\xi_{t+1}|D_t) = H_{t+1} - H_t$ is the information gain over  after observing the $t+1$-th step $(s_t, a_t, r_t, s_{t+1})$
Equation \ref{equ:info_coeff} essentially depicts that as the new information gets smaller with more explorations, meaning that the exploration has already gained a lot of information and getting new information is hard, the bound of the estimation error for value function gets smaller. As long as $\Gamma_{t^\dagger}(\delta)$ is finite, the estimation error can be bounded.

To ensure the stability of policy updates, in Algorithm \ref{alg:adapted_inner_loop} the context buffer for the reasoning LLM is periodically updated only when enough new information compared with the previous checkpoint is obtained. In the following analysis, we will adopt the same \texttt{Enough-NewInfo} condition as in \cite{liu_reason_2023}, which is $H_k - H_t \ge log2$, indicating there is at least 1-bit of new information obtained.

\subsubsection{Sublinearity in Bayesian Regret}
With the above discussion, we invoke conclusions in RAFA, showing the sublinearity of $\mathcal{R}(T)$.  Our proof follows Appdx.C in \cite{liu_reason_2023}. Here we will briefly recite it with the proof sketch. For detailed proof, readers may refer to \cite{liu_reason_2023}.
\begin{theorem} \label{thm:regret_bound}
    Suppose the reasoning algorithm is $\epsilon$-optimal and $\max_s V_\theta(s) \le L$, under Assumption \ref{assump:prior_same}. With $\epsilon = O(\frac{1}{\sqrt{T}})$ and $\delta = O(\frac{1}{\sqrt{T}})$, the Bayesian regret of the reasoning agent satisfies $\mathcal{R}(T) = O(\sqrt{T})$
\end{theorem}
\begin{proof}
   
    According to Eq. C.1 in \cite{liu_reason_2023}, the Bayesian regret in Eq. \ref{equ:bayesian_regret} can be decomposed as:
    \begin{align} \label{equ:regret_decomp_1}
        \mathcal{R}(T) & = \mathbb{E}_{\theta \sim p(\theta)}[\sum_{t=1}^{T} V^{\pi^*}_\theta (s_t) - V^{\pi^t}_\theta (s_t)] \notag \\
        & = \mathbb{E}_{\theta \sim p(\theta)}[\sum_{k=1}^{K} \sum_{t=t_k}^{t_{k+1}-1} V^{\text{PL}^*(\theta)}_{\theta} (s_t) - V^{\pi^t}_\theta (s_t)] \notag \\
        & = \underbrace{\mathbb{E}_{\theta \sim p(\theta)}[\sum_{k=1}^{K} \sum_{t=t_k}^{t_{k+1}-1} V^{\text{PL}^*(\theta^k)}_{\theta^k} (s_t) - V^{\pi^k}_{\theta^k} (s_t)]}_{\text{term (A): Policy suboptimality}} + \underbrace{\mathbb{E}_{\theta \sim p(\theta)}[\sum_{k=1}^{K} \sum_{t=t_k}^{t_{k+1}-1} V^{\pi^t}_{\theta^k} (s_t) - V^{\pi^t}_\theta (s_t)]}_{\text{term(B): Model estimation gap}}
    \end{align}
    Term(A) in Eq. \ref{equ:regret_decomp_1} recalls the definition of $\epsilon$-optimality (eq. \ref{equ:eps_optim}), leading to:
    \begin{equation}
        \text{term(A)} \le \epsilon T.
    \end{equation}
    Since LLM context is updated periodically, in the following we denote $\theta_k = \theta_t$ and $\pi^k = \pi^t$ for all $t \in \{t_k, t_k+1, ... t_{k+1}-1\}$.
    For term(B) in Eq. \ref{equ:regret_decomp_1}, according to lemma C.1 in \cite{liu_reason_2023}, it can be decomposed into the information gain and value inconsistency, expressed as Eq. C.6 in \cite{liu_reason_2023}:
    \begin{align}\label{equ:regret_decomp_2}
       &\frac{1-\gamma}{\gamma}\cdot\mathbb{E}_{\theta \sim p(\theta)}\Bigl[\sum_{k=0}^{K-1}\mathbb{E}_{\pi^k}\Bigl[\sum_{t=t_k}^{t_{k+1}-1} V_{\theta^k}^{\pi^k}(s_t)- V_{\theta}^{\pi^k}(s_t)\Bigr]\Bigr]\notag\\
       &\quad= \underbrace{\mathbb{E}_{\theta \sim p(\theta)}\Bigl[\sum_{k=0}^{K-1}\mathbb{E}_{\pi^k}\Bigl[\sum_{t=t_k}^{t_{k+1}-1} ({B}_{\theta^k}V^{\pi^k}_{\theta^k})(s_t,a_t) - ({B}_{\theta}V^{\pi^k}_{\theta^k})(s_t,a_t)\Bigr]\Bigr]}_{\text{term (C): information gain}}\notag\\
       &\quad\qquad +\underbrace{\mathbb{E}_{\theta \sim p(\theta)}\Bigl[\sum_{k=0}^{K-1}\mathbb{E}_{\pi^k}\Bigl[\bigl(V_{\theta^k}^{\pi^k}(s_{t_{k+1}}) - V^{\pi^k}_{\theta}(s_{t_{k+1}})\bigr)-\bigl(V_{\theta^k}^{\pi^k}(s_{t_k}) - V^{\pi^k}_{\theta}(s_{t_k}) \bigr) \Bigr]\Bigr]}_{\text{term (D): value inconsistency}},
    \end{align}
    where $B_\theta$ is the Bellman operator satisfying
    \begin{equation}
        B_\theta V^\pi_\theta(s, a) = r_\theta(s, a) + \gamma P_\theta V^\pi_\theta(s, a).
    \end{equation}
    For term(C) in Eq. \ref{equ:regret_decomp_2}, recall Eq. \ref{equ:info_coeff} and Eq. C.10 in \cite{liu_reason_2023}, we have:
    \begin{equation}
        \text{term(C)} \le \sup_{t^\dagger<T}\Gamma_{t^\dagger}(\delta)\cdot \mathbb{E}\bigl[\sqrt{ T (H_0 - H_{T}) }\bigr]+L\delta T
    \end{equation}
    For term(D) in Eq. \ref{equ:regret_decomp_2}, according to Eq. C.11 in \cite{liu_reason_2023}, we have
    \begin{equation}
        \text{term (D)}\le {(4L/{\log2})\cdot\mathbb{E}[H_0 - H_{T}]}+4L\label{eq:b}.
    \end{equation}
    So Bayesian regret is bounded by (Theorem 4.4 in \cite{liu_reason_2023}):
    \begin{equation}\label{equ:regret_order}
        \mathcal{R}(T) = O(\frac{\gamma \cdot (\sup_{t^\dagger<T}\Gamma_{t^\dagger}(\delta) \mathbb{E}\bigl[\sqrt{ (H_0 - H_{T}) }\bigr] \sqrt{T}  + L \delta T + L \mathbb{E}\bigl[\sqrt{ (H_0 - H_{T}) }\bigr])}{1-\gamma}  + \epsilon T)
    \end{equation}
    According to Eq. \ref{equ:regret_order}, if we have $\epsilon = O(\frac{1}{\sqrt{T}})$ and $\delta = O(\frac{1}{\sqrt{T}})$, then the whole term is essentially $O(\sqrt{T}).$
\end{proof}

As the number of iterations increases, the policy is expected to approach optimality. This is due to the posterior distribution increasingly concentrating around the true environmental parameter $\theta$ with each new observation, which in turn refines the model's estimation. With a more accurate model and the assumption of planning optimality, we predict the policy will eventually align with the optimal one. This convergence is further assured by Assumption \ref{assump:oracle_judge}, which posits that the judge module provides perfect feedback, thus enabling the agent to arrive at the correct answers. The $\mathbb{E}[\sqrt{H_0 - H_T}]$ term on the right-hand side of Equation \ref{equ:regret_order} indicates that when the information gain is small after exploration, suggesting that the underlying knowledge environment is simple, the Bayesian regret bound becomes tighter. Besides, it's important to note that this analysis does not account for discrepancies between the knowledge environment and the actual environment, an aspect relevant to the outer loop and designated for future investigation.

When considering other paradigms as outlined in Section \ref{sec:framework_inner}, one-shot methods are limited by a lack of iterative feedback and hence retain a constant uncertainty in model parameter estimation, precluding convergence. Similarly, iterative methods that do not engage in Bayesian inference—those unable to refine estimations using data from the knowledge base—are also devoid of performance guarantees due to static uncertainty. Therefore, it is the iterative reasoning algorithms that integrate progressively expanding information which are endowed with provable efficiency. The specific design choices within this paradigm will be addressed in subsequent discussions.

\subsection{Observation Quality}
In the context of a knowledge environment considered as noisy, the data's quality is crucial for the reasoning agent's Bayesian inference and the resulting precision in model parameters. Assessing the influence of the knowledge base type—LLM versus KG—on reasoning, we note from an information-theoretic view that KGs, with their structured representation of knowledge, typically produce less uncertain responses than LLMs. The latter introduces greater uncertainty through its training on diverse datasets and in-context generation of responses. This disparity in response entropy is attributed to the inherent design differences between the deterministic retrieval from KGs and the probabilistic output from LLMs. When observational data is sparse, it usually exacerbates the uncertainty in the estimation of model parameters, leading to policies that may be suboptimal during model-based planning.

\subsection{LLM as Reasoning Agent}
Using an LLM for Bayesian model-based planning is key to achieving convergence, as its simultaneous functioning as a model, actor, and critic is essential for the planning's effectiveness. In contrast, planning methods that omit Bayesian inference will struggle with persistent uncertainty, obstructing the path to convergence. Additionally, employing a basic language model instead of an LLM as the prior requires a greater number of observations for the model to refine its parameter estimations and approach accuracy.

\subsection{Real-world Feedback Incorporation}
User feedback is a vital form of real-world data that enhances the knowledge environment, guiding it to refine its reasoning capabilities. Knowledge Graphs (KGs), with their structured format, adeptly utilize this detailed feedback. They can directly incorporate corrections or new insights into their interconnected nodes, aligning closely with user-provided data.
On the other hand, LLMs face more complexity when integrating feedback. As LLMs are trained on extensive datasets, incorporating new user feedback involves a retraining or fine-tuning process that might not preserve the feedback's full detail due to the data processing inequality. This principle suggests that with each transformation step, such as integrating feedback, there's a risk of losing some useful information. For LLMs, this means the nuanced feedback may get diluted as it's merged with pre-existing data patterns, potentially introducing noise and lessening the impact of the updates.
While both KGs and LLMs can improve from user interactions, KGs are typically better at preserving the accuracy of user feedback. LLMs may inadvertently alter feedback due to their complex update mechanisms. Consequently, KGs often provide a more dependable approach for incorporating new user insights into the knowledge environment.

\newpage
\bibliography{references}

\newpage
\appendix
\onecolumn

\end{document}